\documentclass{article}

\usepackage{microtype}
\usepackage{graphicx}
\usepackage{subcaption}
\usepackage{booktabs} 
\usepackage{algorithm}
\usepackage{float}
\usepackage{bm}
\usepackage{makecell}

\newcommand{\R}{\mathbb{R}}

\newcommand{\E}{\mathbb{E}}

\newcommand{\etph}{\texttt{ETPH}}
\newcommand{\nlt}{N_{\texttt{large}}}

\usepackage{hyperref}

\usepackage[preprint]{icml2026}

\usepackage{amsmath}
\usepackage{amssymb}
\usepackage{mathtools}
\usepackage{amsthm}

\usepackage[capitalize,noabbrev]{cleveref}

\theoremstyle{plain}
\newtheorem{theorem}{Theorem}[section]

\newtheorem{lemma}{Lemma}[section]

\theoremstyle{definition}
\newtheorem{definition}{Definition}

\theoremstyle{remark}

\usepackage[textsize=tiny]{todonotes}

\begin{document}

\twocolumn[
  \icmltitle{Multi-Objective Reinforcement Learning for Large-Scale Tote Allocation in Human-Robot Collaborative Fulfillment Centers}



  \icmlsetsymbol{equal}{*}

  \begin{icmlauthorlist}
    \icmlauthor{Sikata Sengupta}{yyy,equal}
    \icmlauthor{Guangyi Liu}{comp,equal}
    \icmlauthor{Omer Gottesman}{comp}\\
    \icmlauthor{Joseph W Durham}{comp}
    \icmlauthor{Michael Kearns}{yyy,comp}
    \icmlauthor{Aaron Roth}{yyy,comp}
    \icmlauthor{Michael Caldara}{comp}
  \end{icmlauthorlist}
  
  \icmlaffiliation{yyy}{Department of Computer and Information Science, University of Pennsylvania, Philadelphia, PA, USA}
  \icmlaffiliation{comp}{Amazon}

  \icmlcorrespondingauthor{Sikata Sengupta}{sikata@seas.upenn.edu}

  \icmlkeywords{Machine Learning, ICML}

  \vskip 0.3in
]



\printAffiliationsAndNotice{\icmlEqualContribution}

\begin{abstract}
 Optimizing the consolidation process in container-based fulfillment centers requires trading off competing objectives such as processing speed, resource usage, and space utilization while adhering to a range of real-world operational constraints. This process involves moving items between containers via a combination of human and robotic workstations to free up space for inbound inventory and increase container utilization. We formulate this problem as a large-scale Multi-Objective Reinforcement Learning (MORL) task with high-dimensional state spaces and dynamic system behavior. Our method builds on recent theoretical advances in solving constrained RL problems via best-response and no-regret dynamics in zero-sum games, enabling principled minimax policy learning. Policy evaluation on realistic warehouse simulations shows that our approach effectively trades off objectives, and we empirically observe that it learns a single policy that simultaneously satisfies all constraints, even if this is not theoretically guaranteed. We further introduce a theoretical framework to handle the problem of error cancellation, where time-averaged solutions display oscillatory behavior. This method returns a single iterate whose Lagrangian value is close to the minimax value of the game. These results demonstrate the promise of MORL in solving complex, high-impact decision-making problems in large-scale industrial systems.
\end{abstract}

\section{Introduction and Background}

Modern warehouses (\Cref{fig:consolidation_stations}, left) increasingly rely on human-robot collaboration to manage large-scale inventory operations. For example, Amazon's robotic fulfillment system, ``Sequoia,'' leverages these collaborations to store inventory $75\%$ faster and process customer orders $25\%$ faster than its prior generation fulfillment system through the integration of multiple robotic products used to containerize inventory and present it at an ergonomic height for processing at an employee workstation. Key to unlocking these speed-ups is efficient resource management within the storage system. In this work, we focus on a particular space management process called ``consolidation,'' wherein items are moved between containers to improve utilization (\Cref{fig:consolidation_stations}, right) and free up space for new inbound items. Items are stored in containers, known as ``totes'', which are cyclically inbounded into storage shelves, consolidated to reclaim space, and picked from to fulfill customer orders. Efficient tote selection for consolidation is thus central to optimizing both storage utilization and throughput in this environment, but requires careful management of multiple operational objectives.

For illustration purposes, we assume the following stylized consolidation process: (1) totes can be processed by both humans and robots, each with different manipulation capabilities. We assume that humans are capable of safely manipulating any item, whereas robots are only capable of safely manipulating some items. The task of both operators is to transfer items from a ``source'' tote to one or more ``destination'' totes. Source totes are ejected when emptied, and destination totes are ejected when full. A range of factors influence whether a tote is better suited for a human or robot to consolidate including the particular items in the source tote, the destination tote fullness, and the tote properties (e.g., we assume two types of totes ``smaller'' and ``larger'' of differing dimensions). Employee workstations typically handle a broader range of item types, but use up processing capacity from other workflow, whereas robotic stations excel at consistency but may be restricted by item properties. As such, consolidation decisions must not only maximize throughput and space utilization, but also \textit{intelligently allocate workloads between humans and robots} to leverage their respective strengths. Operational goals for consolidation are similarly multi-faceted and often conflicting. Key performance indicators (KPIs) include throughput efficiency, inventory tote type balance, and space utilization. Optimizing across these dimensions requires managing complex trade offs under dynamic conditions and real-world constraints.

Due to the complexity of this decision, we expect that heuristics and single-objective optimization strategies will fail to generalize or scale beyond some small number of trade offs. A common approach to handle such scenarios is \textit{scalarization}, where the different objectives are combined into a single objective function using predefined weights. However, in one-shot settings, these weights must be fixed in advance, making the approach sensitive to weight selection and ill-suited for environments with shifting priorities. Moreover, they typically optimize for one KPI at the expense of others, resulting in suboptimal system-wide performance. In contrast, our work avoids the need for manual weight specification by employing an adaptive procedure that iteratively adjusts the objective weights during training. This enables the learning process to naturally converge toward a minimax solution that balances objectives under dynamic and uncertain conditions, as we elaborate later in the paper.

In this work, we address this problem by framing the tote consolidation and allocation problem as a large-scale \textit{Multi-Objective Reinforcement Learning} (MORL) task. Our formulation captures the dynamic and stochastic nature of warehouse operations while explicitly modeling the trade offs between conflicting KPIs and the heterogeneous capabilities of human and robotic stations. Building on recent theoretical advances in constrained RL, particularly best-response and no-regret learning in zero-sum games, we develop a principled and scalable approach for minimax policy learning in complex multi-objective settings.  

Our contributions are threefold: (i) we propose a novel MORL formulation for real‑world consolidation problems in human–robot collaborative fulfillment centers, explicitly modeling heterogeneous station capabilities; (ii) we develop a theoretical framework that reformulates the multi‑objective problem as a zero‑sum Lagrangian game and prove that we can select a single iterate from the time‑averaged approximate minimax mixture whose Lagrangian value is close to the minimax value of the game; and (iii) we demonstrate strong empirical performance on realistic warehouse simulations, where our approach outperforms baselines across KPIs. These results provide compelling evidence that MORL is a viable and impactful approach to optimizing high‑dimensional, high‑stakes industrial decision making systems.

\begin{figure}[t]
  \vskip 0.2in
  \begin{center}
    \begin{minipage}{\linewidth} 
        \centering
        \begin{minipage}[b]{0.41\linewidth}
            \centering
            \includegraphics[width=\linewidth]{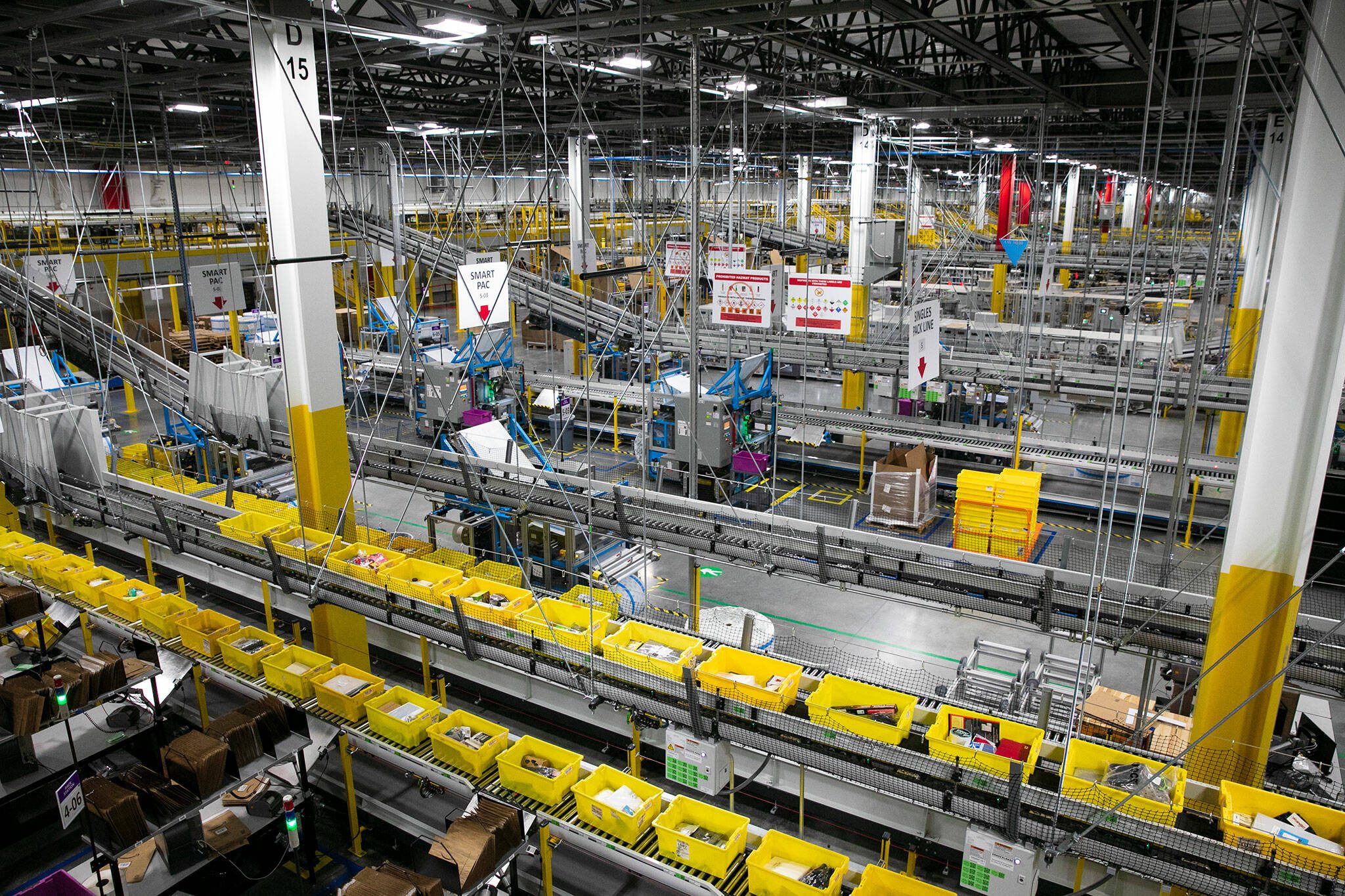}
        \end{minipage}
        \hfill
        \begin{minipage}[b]{0.48\linewidth}
            \centering
            \includegraphics[width=\linewidth]{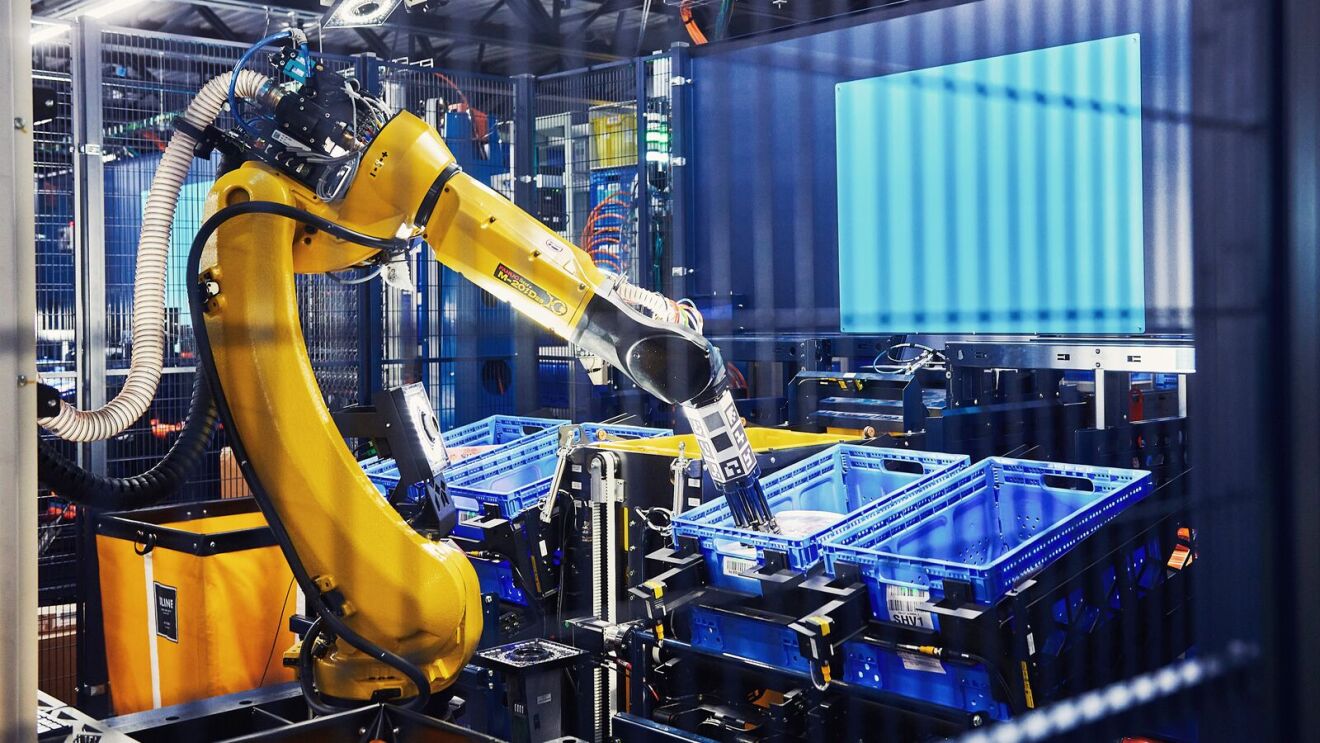}
        \end{minipage}
    \end{minipage}
    \caption{
      Left: High-throughput container-based fulfillment center with dense human-robot collaboration. Right: Robotic consolidation station autonomously transferring items between totes.
    }
    \label{fig:consolidation_stations}
  \end{center}
\end{figure}

\section{Related Work}
We organize our related work mainly into two categories--theoretical works related to multi-objective or constrained RL, and empirical works related to this domain of large-scale industrial applications like inventory management. 
 
\citet{hayes2022practical} provides a comprehensive overview of  general multi-objective RL approaches that have been explored. As noted in \cite{yang2019generalized}, MORL methods can often be divided into two main categories: single-policy and multiple-policy methods. Single policy methods often require knowledge of preferences in advance, so in this paper we take a multiple policy approach. Constrained RL problems can be modeled as a multi-objective problem with constraint violation terms written as additional objectives. Constrained RL is of interest in many domains ranging from safety to fairness, where one wants to learn a policy that optimizes some global reward subject to various trustworthiness constraints. \cite{calvo}  utilize primal-dual methods (via optimizing single objective RL for the learner) and utilizing dual descent for the regulator. Especially for zero-sum games, one can relate game-theoretic notions of best-response and no-regret to primal-dual style algorithms. \cite{miryoosefi2019reinforcement} use best-response vs. no-regret dynamics with convex constraints and provide safety guarantees on the time-averaged strategies of each player. \cite{eaton2025intersectionalfairnessreinforcementlearning} provide oracle-efficient algorithms similarly utilizing best-response vs. no-regret dynamics to solve for a minimax solution of the Lagrangian of this problem even when the number of objectives (groups in the context of fairness) is exponentially large for both tabular and large state space MDPs. We choose to utilize this repeated game framework because it does not require us to know preferences over objectives in advance and because it does not make strong assumptions on the structure/geometry of multi-objective problem.

On the application side, \cite{el2023multi} consider a Multi-Objective RL approach to sustainable supply-chain optimization. They learn policies to try to approximate the Pareto Frontier. They make use of an outer-inner loop procedure where the outer loop makes use of evolutionary strategy updates and the inner loop uses Envelope Multi-Objective Q-Learning \cite{yang2019generalized}.

\section{Problem Setting and Model} \label{sec:problem_setting_and_model}

\begin{figure}[t]
  \vskip 0.2in
  \begin{center}
    \centerline{\includegraphics[width=\columnwidth]{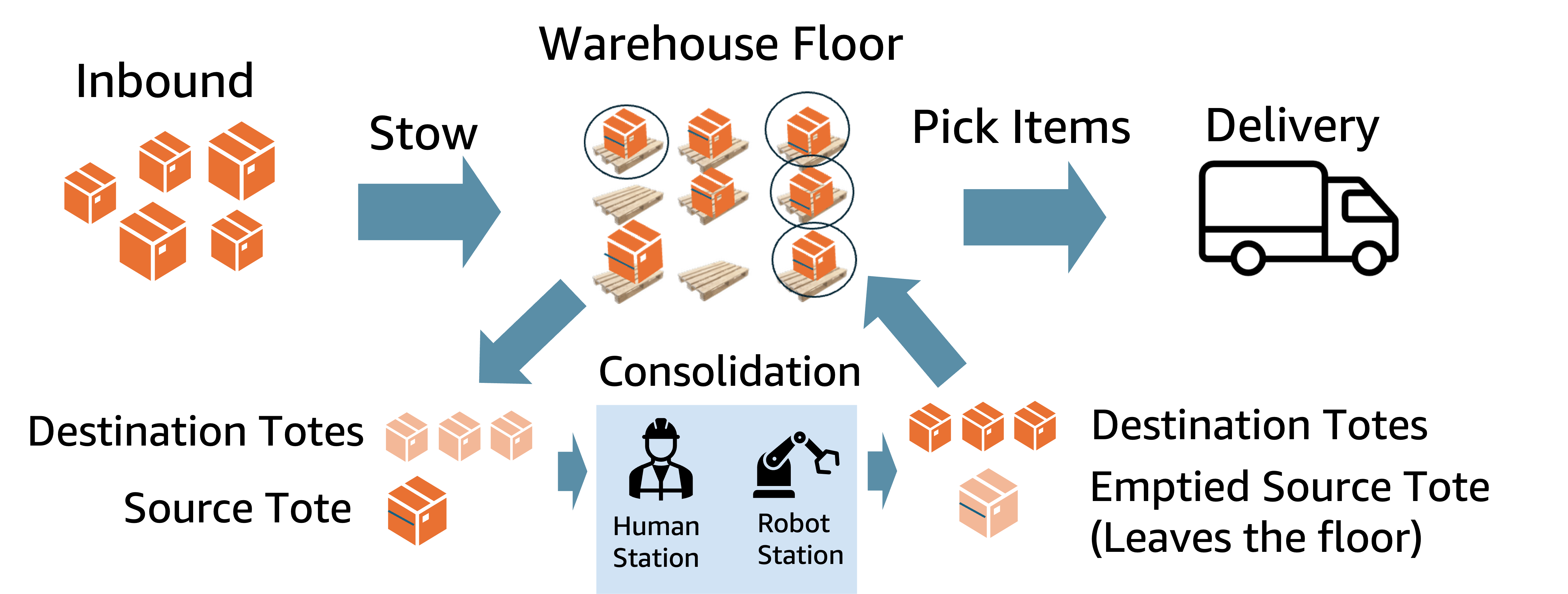}}
    \caption{Human-robot collaborative fulfillment center workflow.}
    \label{fig:overview}
  \end{center}
\end{figure}

We study the decision-making problem underlying tote consolidation in human-robot collaborative fulfillment centers, illustrated in \Cref{fig:overview}. Consolidation operates alongside inbound and pick processes and involves transferring items from partially filled source totes into destination totes to reclaim storage capacity. Employee and robot stations differ in their manipulation capabilities: employee stations can handle complex or irregular items, while robot stations are limited in their reach, perception, and handling of deformable and reflective products. Effective consolidation therefore requires allocating work across heterogeneous stations while respecting operational constraints.

At each decision point, the system selects a source tote and a compatible destination tote based on tote-level attributes such as fill level and item complexity, as well as floor-level signals, such as the current distribution of tote types. The objective is to balance multiple, often competing, KPIs, including maintaining high overall processing efficiency, making effective use of available empty capacity during consolidation, and preserving a healthy distribution of tote types due to the warehouse configuration. We model this as a multi-objective sequential decision-making problem and address it using constrained multi-objective reinforcement learning (MORL), which enables the learning of station-aware tote prioritization policies that adapt to dynamic warehouse conditions.

\subsection{Preliminaries}  
We consider an episodic RL setting modeled as a Markov Decision Process (MDP) 
$\mathcal{M} = \big(\mathcal{X}, \mathcal{A}, \{P_h\}_{h=0}^{H-1}, \{r_{i,h}\}_{i=0,\,h=0}^{m,\,H-1}, \mu, H \big)$, where $\mathcal{X}$ is the state space, $\mathcal{A}$ is the action space,  
$P_h: \mathcal{X} \times \mathcal{A} \rightarrow \Delta(\mathcal{X})$ is the transition function at step $h$,  
$r_{i,h}: \mathcal{X} \times \mathcal{A} \rightarrow [0, R]$ is the reward function for objective $i \in \{0, \dots, m\}$ at step $h$,  
$\mu$ is the initial-state distribution, and $H$ is the episode horizon. Here, $\Delta(\mathcal{X})$ denotes the probability simplex over $\mathcal{X}$. A non-stationary policy is a sequence $\pi = \{\pi_h\}_{h=0}^{H-1}$,  
where each $\pi_h: \mathcal{X} \rightarrow \Delta(\mathcal{A})$ maps states to distributions over actions.  
We use $[H] := \{0, \dots, H-1\}$ to denote the set of time indices.  
Given a policy $\pi$, the trajectory $(\mathbf{x}_0, \mathbf{a}_0, \dots, \mathbf{x}_{H-1}, \mathbf{a}_{H-1})$ is generated by  
$\mathbf{x}_0 \sim \mu$, $\mathbf{a}_h \sim \pi_h(\mathbf{x}_h)$, and $\mathbf{x}_{h+1} \sim P_h(\cdot \mid \mathbf{x}_h, \mathbf{a}_h)$ for all $h \in [H-1]$. For objective $i$, the value function under policy $\pi$ is
\[
V_{i,h}^\pi(x) \;=\; \mathbb{E}^\pi \left[ \sum_{t=h}^{H-1} r_{i,t}(\mathbf{x}_t, \pi_t(\mathbf{x}_t)) \,\bigg|\, \mathbf{x}_h = x \right].
\]

We cast our setting as a constrained RL problem \cite{paternain2019constrained,eaton2025intersectionalfairnessreinforcementlearning} of the form
\begin{equation}    \label{eq:contrained_rl}
    \begin{aligned}
        \max_{D \in \Delta(\Pi)} \quad & 
    \mathbb{E}_{\pi \sim D} \left[ \mathbb{E}_{\mathbf{x} \sim \mu} \left[ V_{0,0}^\pi(\mathbf{x}) \right] \right] \\
        \text{s.t.} \quad 
        & \mathbb{E}_{\pi \sim D} \left[ \mathbb{E}_{\mathbf{x} \sim \mu} \left[ V_{i,0}^\pi(\mathbf{x}) \right] \right] \le \alpha_i,
        \end{aligned}
\end{equation}
for all $i \in \{1, \dots, m\}$, where $\Pi$ is the policy class and $D \in \Delta(\Pi)$ denotes a distribution over policies.  
Here, $V_{i,0}^\pi(\mathbf{x})$ is the value function at the initial time step for objective $i$, with $i=0$ denoting the primary objective and $i \ge 1$ corresponding to constraint objectives. 

To solve the constrained optimization problem \eqref{eq:contrained_rl}, we introduce its Lagrangian:
\begin{equation}
    \begin{aligned}
        &\mathcal{L}(D,\lambda) 
= \mathbb{E}_{\pi \sim D} \left[ \mathbb{E}_{\mathbf{x} \sim \mu} \left[ V_{0,0}^\pi(\mathbf{x}) \right] \right]\\ 
        & \hspace{1cm} + \sum_{i=1}^m \lambda_i \left( \alpha_i - \mathbb{E}_{\pi \sim D} \left[ \mathbb{E}_{\mathbf{x} \sim \mu} \left[ V_{i,0}^\pi(\mathbf{x}) \right] \right] \right),
    \end{aligned}
\end{equation}
where $\lambda \in \mathbb{R}^m_{\ge 0}$ are the Lagrange multipliers. This formulation can be interpreted as a zero-sum game between:  
(i) a \emph{Learner} selecting a distribution over policies $D \in \Delta(\Pi)$ to maximize $\mathcal{L}$, and  
(ii) a \emph{Regulator} selecting nonnegative multipliers to minimize $\mathcal{L}$.  
The regulator’s strategy space is the compact set $\Lambda \;=\; \left\{ \lambda \in \mathbb{R}^m_{\ge 0} \;:\; \|\lambda\|_1 \le C \right\}.$ Under these conditions, by~\cite{sion1958general}'s minimax theorem, 
\[
 L^* \coloneqq \min_{\lambda \in \Lambda} \; \max_{D \in \Delta(\Pi)} \; \mathcal{L}(D,\lambda) = \max_{D \in \Delta(\Pi)}\min_{\lambda \in \Lambda}\mathcal{L}(D,\lambda) .
\]

Note that because the learner’s objective is linear in D, a best-response to any fixed $\lambda_t$ is attained at an extreme point of $\Delta(\Pi)$, and thus each $D_t$ corresponds to a single policy. We can define an approximate minimax equilibrium as follows.

\begin{definition}[$\nu$-approximate Minimax Equilibrium]   \label{def:minimax_equilibrium}
    $(\hat{D},\hat{\lambda})$ is a $\nu$-approximate minimax equilibrium if:
    \[
        \max_{D \in \Delta(\Pi)} \mathcal{L}(D,\hat{\lambda}) - \nu \leq \mathcal{L}(\hat{D},\hat{\lambda}) \leq \min_{\lambda \in \Lambda} \mathcal{L}(\hat{D},\lambda) + \nu.
     \]
\end{definition}

To compute an approximate solution (as is shown in \cref{def:minimax_equilibrium}) to the minimax problem \(\min_{\lambda \in \Lambda} \max_{D \in \Delta(\Pi)} \mathcal{L}(D,\lambda)\) in practice, we adopt a repeated game perspective.  
In this view, the learner and the regulator iteratively update their strategies over $T$ rounds, and the average strategies converge to an approximate equilibrium.  
Our approach builds on the seminal best-response vs. no-regret result of \cite{Freund1996GameTO}:

\begin{theorem}[(Informal) Best-Response vs. No-Regret \cite{Freund1996GameTO}]\label{thm:best_response_no_regret}
    Let $(D_1,\ldots,D_T)$ be the sequence of best-response (distributions over) policies and let $(\lambda_1,\ldots,\lambda_T)$ be the corresponding sequence of Lagrangian weights maintained by the learner and regulator respectively. Then $(\bar{D},\bar{\lambda})$ form an approximate minimax equilibrium of the game defined by $\mathcal{L}$, where $\bar{D} = \frac{1}{T}\sum_{t=1}^T D_t$ and $\bar{\lambda} = \frac{1}{T}\sum_{t=1}^T \lambda_t$. 
\end{theorem}

We next quantify the learner’s ability to produce an approximate best-response policy, following~\cite{eaton2025intersectionalfairnessreinforcementlearning}:

\begin{definition}[$\epsilon$-approximate best-response ]\label{def:epsilon_best_response}
    A distribution over policies $D^*$ is an $\epsilon$-best-response if for a given $\lambda \in \Lambda$, 
    \[
    \mathcal{L}(D^*,\lambda) \geq \max_{D \in \Delta(\Pi)} \mathcal{L}(D,\lambda)-\epsilon.
    \]
\end{definition}

Given a fixed $\lambda \in \Lambda$ from the regulator, the learner’s best-response problem can be expressed as optimizing a scalarized reward function~\cite{eaton2025intersectionalfairnessreinforcementlearning}:
\[
    r_\lambda(x,a) = r_0(x,a) + \sum_{i=1}^m \lambda_i \left( \frac{\alpha_i}{H} - r_i(x,a) \right).
\]
Thus, in practice, the best-response reduces to solving a single-objective RL problem parameterized by $\lambda$. In this paper, we implement this step using Deep Q-Learning (DQN) \cite{mnih2015human}, although many other deep RL algorithms \cite{mnih2013playing, schulman2017proximal,haarnoja2018soft} could be used in its place.


For the regulator, we define its regret with respect to $\Lambda$ as follows:
\begin{equation}    \label{eq:regulator_regret}
    \text{Reg}(\lambda_{1:T})=\sum_{t=1}^T \mathcal{L}(D_t,\lambda_t)-\min_{\lambda \in \Lambda}\sum_{t=1}^T \mathcal{L}(D_t,\lambda).
\end{equation}
If the regret of an algorithm is sublinear in $T$, equivalently, if the average regret tends to $0$ as $T \to \infty$, we refer to it as a \textit{no-regret} algorithm throughout this paper. One concrete example is Online Gradient Descent (OGD):


\begin{theorem}[Online Gradient Descent No-Regret \cite{zinkevich}]\label{thm:ogd_no_regret}
    For a sequence of differentiable functions $f^1,\ldots,f^T:S \rightarrow \mathbb{R}$ such that $\|\nabla f^t(x)\| \leq G$ for any $1 \leq t \leq T, x \in S$, i.e., $G$ is an upper bound on the gradient magnitudes. Then, for the following sequence $x^1,\ldots,x^T$ built according to 
    \[
    x^{t+1} = \Pi_S(x^t-\eta \nabla f^t(x^t)),
    \]
    where $\eta = \frac{D}{G \sqrt{T}}$ and $D$ is an upper bound on the diameter of the closed convex set $S$, the regret is bounded as:
    \[
    \sum_{t=1}^T f^t(x^t) \leq \min_{x \in S}\Big(\sum_{t=1}^Tf^t(x)\Big)+DG\sqrt{T}.
    \]
\end{theorem}

Having introduced best-response and no-regret strategies in \cref{def:epsilon_best_response} and \eqref{eq:regulator_regret}, it is important to note that the guarantee in \cref{thm:best_response_no_regret} applies to the performance of the time-averaged solution $\bar{D}$ in expectation. As a result, individual policies drawn from the support of $\bar{D}$ may violate constraints, as long as such violations are offset on average across the mixture. Consequently, closeness of the mixture to the minimax value does not imply that any single policy in its support satisfies all constraints.

To illustrate this phenomenon, consider a simple MDP with state space $\mathcal{X}=X_{\text{safe}}\cup X_{\text{left}}\cup X_{\text{right}}$, where deviations to either side correspond to entering a danger region. While an ideal policy would remain in the safe region at all times, a mixture policy could alternate between left and right, consistently violating constraints yet appearing feasible in expectation due to cancellation. In Appendix \ref{app:err_canc}, we build on the framework of \citet{eaton2025intersectionalfairnessreinforcementlearning} to reformulate the problem and show that, when the mixture is near-minimax, it is nevertheless possible to probabilistically extract a single iterate whose Lagrangian value is close to the minimax value. With this theoretical motivation in place, we now describe the MDP model used to capture the dynamics of the fulfillment center.

\subsection{MDP Model of Warehouse Floor}
We begin by defining the state space, followed by the action space, transition dynamics, and reward structure.

\vspace{2mm}

The warehouse floor $F$ has capacity for $\texttt{floor\_max}$ totes, which we assume to be large. Each state $x \in \mathcal{X}$ corresponds to a specific tote slot, which may be occupied or empty, together with aggregate information describing the global state of the floor. We represent $x$ as a feature vector comprising the following task relevant quantities:
\begin{equation}    \label{eq:state_def}
    \begin{aligned}
&x_t = \Big(N_{\texttt{large}}, \texttt{ETPH}, L^{H}_{S}, L^{H}_{D},L^{R}_{S}, L^{R}_{D}, \\ & \hspace{2cm} O_{k}, \texttt{LTE}, n_{\texttt{item}}, n_{\texttt{pick}}, \texttt{GCU},t \Big)
\end{aligned}
\end{equation}
We describe each component below:

\noindent \textbullet\ $N_{\texttt{large}} \in \mathbb{Z}_{+}$: Total number of larger totes on the floor.\\
\noindent \textbullet\ $\texttt{ETPH} \in \R_+$: Empty Totes Per Hour, which measures the number of source totes emptied via consolidation per hour. \\
\noindent \textbullet\ $L^{H}_{S}, L^{H}_{D}, L^{R}_{S}, L^{R}_{D} \in \mathbb{Z}_{+}$: 
Queue lengths of source $(S)$ and destination $(D)$ totes at human $(H)$ and robotic $(R)$ stations.\\
\noindent \textbullet\ $O_{k}$: Occupancy; 
$0$ if empty, $1$ if occupied by a larger tote, $2$ if occupied by a smaller tote.\\
\noindent \textbullet\ $n_{\texttt{item}} \in \mathbb{Z}_{+}$: Number of items in the tote.\\
\noindent \textbullet\ $\texttt{LTE} \in (0,1]$: Likelihood to be successfully emptied by robot stations, approximated by $1/n_{\texttt{item}}$.\\
\noindent \textbullet\ 
$n_{\texttt{pick}} \in \mathbb{Z}_{+}$: Number of items in the tote scheduled to be picked that day.\\
\noindent \textbullet\ $\texttt{GCU} \in \R_+$: Gross Cubic Utilization, approximated as the estimated rectangular volume of all items in the tote.\\
\noindent \textbullet\ $t \in [0, H]$: Current time step in the episode.

\vspace{2mm}

The action space is given by
\[\scalebox{0.85}{$
\mathcal{A} =
\left\{
\begin{array}{c}
\texttt{ignore} \\
\texttt{not\_ignore}
\end{array}
\right\}\times
\left\{
\begin{array}{c}
\texttt{source} \\
\texttt{destination}
\end{array}
\right\}\times
\left\{
\begin{array}{c}
\texttt{human} \\
\texttt{robot}
\end{array}
\right\},
$}\]
which encodes the decision for a given tote slot, to either skip it or assign the tote it contains to a human or robotic station, and to designate it as either a source or destination tote.

State transitions are determined by the environment's update rules for floor and station states. Specifically:
(i) the number of larger totes decreases when one is sent as a source tote to a station and increases when new larger totes are stored on the floor;
(ii) $\etph$ evolves according to the aggregate outcomes of consolidation operations across all stations, depending on the assigned station type, tote composition, and other operational factors;
(iii) station queue lengths are updated based on the action’s station and role assignment.
After each action, the next floor slot is selected via linear traversal of the remaining slots. At the end of each day, the stow and pick processes (by which new totes are added to the floor and items from totes are removed to fulfill order demands) are simulated and we shuffle the tote orderings. 

\vspace{2mm}

The reward functions $\{r_i\}_{i=0}^m$ correspond to the key performance indicators (KPIs) of interest, such as $\etph$, $\nlt$, Source to Destination tote ($S/D$) ratio, and consolidation station capacity compliance. Some rewards directly capture operational constraints specified in our optimization problem. The primary objective reward is
\begin{equation}    \label{eq:etph_reward}
    r_0(x_t,a_t,x_{t+1}) = x_{t+1}[\etph],
\end{equation}
which captures the throughput efficiency achieved by the current consolidation decision. Since in our setting, the state includes some floor-wide summary statistics , we use $\etph$ as both part of the state space and the reward signal $r_0$ we are optimizing over. We use the shorthand $x_t[\etph]$ to corresponding to the value of that floor-wide statistic at time $t$ in state $x_t$. The first constraint-based reward is
\begin{equation}    \label{eq:nlt_constraint}
    r_1(x_t,a_t,x_{t+1})=
\frac{1}{H}
\left(
\alpha_1
-
\frac{x_{t+1}[\nlt]}{\texttt{floor\_max}}
\right),
\end{equation}
which penalizes deviations from a desired fraction of larger totes on the floor, reflecting structural constraints imposed by the fulfillment center layout. The second constraint-based reward is
\begin{equation}
\label{eq:cap_constraint}
\scalebox{1}{$
r_2(x_t,a_t,x_{t+1}) = \frac{1}{H} \left( -\alpha_2 + \frac{x_{t+1}[L^H_S] + x_{t+1}[L^R_S]}{1 + x_{t+1}[L^H_D] + x_{t+1}[L^R_D]} \right)
$}
\end{equation}
which encourages a balanced $S/D$ assignment across stations. By limiting excessive usage of destination totes, this constraint reduces downstream swapping and improves overall consolidation efficiency. The third and fourth constraint-based rewards are: 
\begin{equation}    \label{eq:human_Q}
r_3(x_t,a_t,x_{t+1})
=
\frac{1}{H}
\left(
\alpha_3
-
\big(x_{t+1}[L^H_S] + x_{t+1}[L^H_D]\big)
\right),
\end{equation}
\begin{equation}    \label{eq:robot_Q}
r_4(x_t,a_t,x_{t+1})
=
\frac{1}{H}
\left(
\alpha_4
-
\big(x_{t+1}[L^R_S] + x_{t+1}[L^R_D]\big)
\right),
\end{equation}
which enforce capacity constraints at human and robotic stations by penalizing excessive queue lengths, thereby preventing overload and ensuring stable operation.

We aim to maintain the Markov property in our MDP formulation. However, in practice, certain modeling simplifications, which introduced to reduce the dimensionality of the state and action spaces, may introduce mild deviations. We formulate the following constrained RL problem based upon the MDP and problem described above.
\begin{equation}    \label{eq:constrained_rl}
    \begin{aligned}
    \max\limits_{D \sim \Delta(\Pi)} & \; \E_{\pi \sim D} \Big[ \sum_{t=0}^{H-1} \etph_t(\pi) \Big] \\
    \text{s.t. } & \E_{\pi \sim D} \Big[\sum_{t=0}^{H-1} N_{\texttt{large},t}(\pi) \Big] \leq \alpha_1, \\
    &\E_{\pi \sim D} \Big[ -\sum_{t=0}^{H-1} (S/D)_t(\pi) \Big] \leq -\alpha_2, \\
    &\E_{\pi \sim D} \Big[ \sum_{t=0}^{H-1} \big(L^{H}_{S,t}(\pi)+L^{H}_{D,t}(\pi)\big) \Big] \leq \alpha_3, \\
     &\E_{\pi \sim D} \Big[ \sum_{t=0}^{H-1} \big(L^{R}_{S,t}(\pi)+L^{R}_{D,t}(\pi)\big) \Big] \leq \alpha_4.
\end{aligned}
\end{equation}

We can label the bounds in \eqref{eq:constrained_rl}, for example, as $\alpha_1 = L$, $\alpha_2 = SD$,  $\alpha_3 = B_H$, and $\alpha_4 = B_R$, where $L$ stands for an upper bound on the number of large totes on the floor, SD stands for a lower bound on the source to destination ratio, and $B_H,B_R$ stand for the human and robotic budget capacities.  
Let $f(\pi) \in \mathbb{R}^m$ denote the vector of constraint values corresponding to policy $\pi$, and let $\alpha \in \mathbb{R}^m$ be the corresponding bound vector, adjusted for the constraint signs. The Lagrangian can then be expressed as
\begin{align*}
    \mathcal{L}(D, \lambda) \;=\; \E_{\pi \sim D} \big[ \etph(\pi)  + \sum_{i=1}^m \lambda_i \big( \alpha_i - f_i(\pi) \big)\big].
\end{align*}

\section{MORL via Best-Response vs.\ No-Regret Dynamics}

To solve the constrained optimization problem in \eqref{eq:constrained_rl}, we adopt a repeated game perspective between two players: a \emph{learner} and a \emph{regulator}. This framework follows the best-response versus no-regret paradigm used in prior work on constrained and multi-objective reinforcement learning \cite{eaton2025intersectionalfairnessreinforcementlearning,calvo,miryoosefi2019reinforcement}. At each round, the regulator specifies a vector of Lagrange multipliers $\lambda_t$, which defines a scalarized reward function. The learner then computes an approximate best-response policy to this reward. The resulting policy is evaluated to estimate constraint satisfaction, and this feedback is used to update the regulator’s multipliers via a no-regret procedure (here we use Online Gradient Descent).

This interaction is repeated over multiple rounds, and the time-averaged strategies of the learner and regulator are returned. Under standard no-regret guarantees, this procedure converges to an approximate minimax equilibrium of the underlying Lagrangian game, yielding policies that balance throughput maximization with constraint enforcement. \cref{alg:repeated_game} summarizes this repeated game procedure, explicitly detailing the interaction between the learner’s best-response updates and the regulator’s no-regret multiplier updates.

\begin{algorithm}[t]
\caption{Repeated Game: Learner vs Regulator}
\begin{algorithmic}[1]
\STATE \textbf{Initialize:} $\lambda_0 \in \Lambda$, total rounds $T$, step size $\eta$
\STATE $\bar{D} \leftarrow 0$, $\bar{\lambda} \leftarrow 0$ \COMMENT{Time averaged strategies}
\FOR{$t = 1$ to $T$}
    \STATE \textbf{Learner:} $D_t \leftarrow \text{DQN\_Best\_Response}(\lambda_{t-1})$
    \STATE $g_t \leftarrow \mathrm{EvaluatePolicy}(D_t)$ \COMMENT{Constraint slacks}
    \STATE \textbf{Regulator:} $\lambda_t \leftarrow \mathrm{Proj}_{\Lambda}(\lambda_{t-1} - \eta g_t)$ \COMMENT{Online gradient descent}
    \STATE $\bar{D} \leftarrow \frac{1}{t}\bigl((t-1)\bar{D} + D_t\bigr)$ \COMMENT{Update time average}
    \STATE $\bar{\lambda} \leftarrow \frac{1}{t}\bigl((t-1)\bar{\lambda} + \lambda_t\bigr)$
\ENDFOR
\STATE \textbf{Return:} $(\bar{D}, \bar{\lambda})$ \COMMENT{Time averaged strategies}
\label{alg:repeated_game}
\end{algorithmic}
\end{algorithm}

\begin{figure}[t]
  \vskip 0.2in
  \begin{center}
    \centerline{\includegraphics[width=\columnwidth]{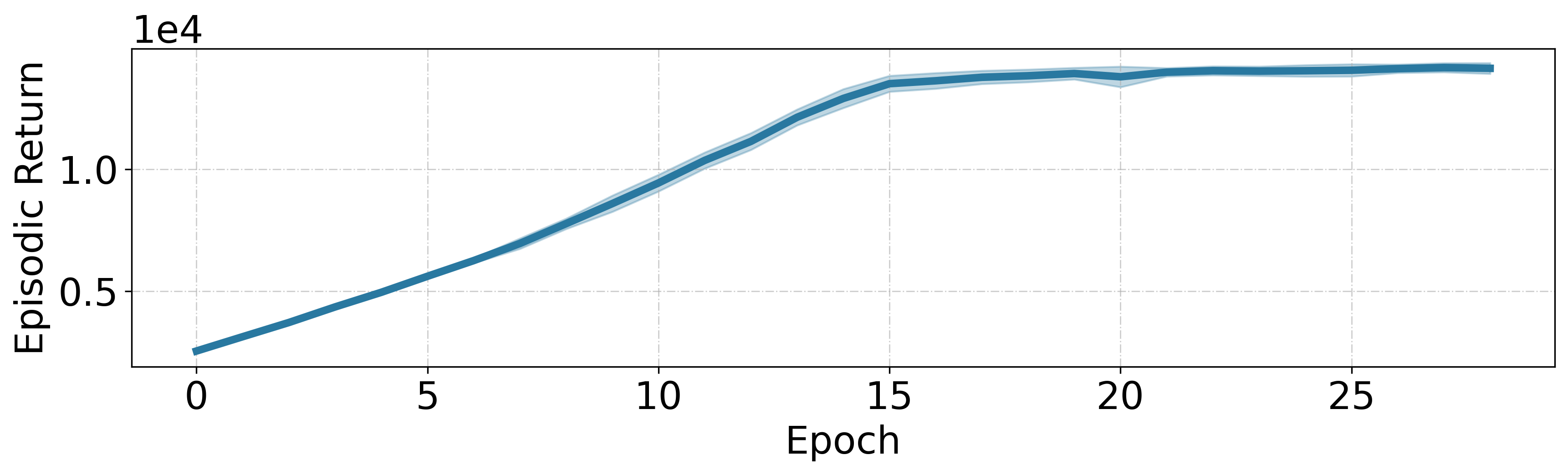}}
    \caption{Episodic returns for DQN in the single-objective ETPH setting, showing unnormalized performance of a best-response policy optimizing ETPH.}
    \label{fig:single_obj}
  \end{center}
\end{figure}

\section{Experiments}

We evaluate our approach using an abstract, event-driven simulator that captures the key operational dynamics of a large-scale human-robot collaborative fulfillment center. The simulator models tote consolidation, station-level capacity constraints, heterogeneous human and robotic capabilities, and evolving floor-level state variables, while abstracting away low-level physical details. This design allows us to isolate the decision-making problem of tote selection and assignment, stress test constraint satisfaction under realistic load conditions, and conduct controlled comparisons across algorithms at scale.

\subsection{Single-Objective  Optimization}

To validate the learning capability of the base RL algorithm, we first evaluate DQN in a single-objective setting using the heuristic ETPH reward. Each simulated day consists of $60 \times 24 = 1440$ decision steps, and a single episode spans $N_{\text{days}} = 10$. The agent is trained for $30$ episodes per run, corresponding to one learner update round in the repeated game procedure. As shown in \Cref{fig:single_obj}, the learned policy exhibits steady performance improvement, with episodic returns increasing substantially over the course of training.

\begin{figure}[t]
  \vskip 0.2in
  \begin{center}
    \centerline{\includegraphics[width=\columnwidth]{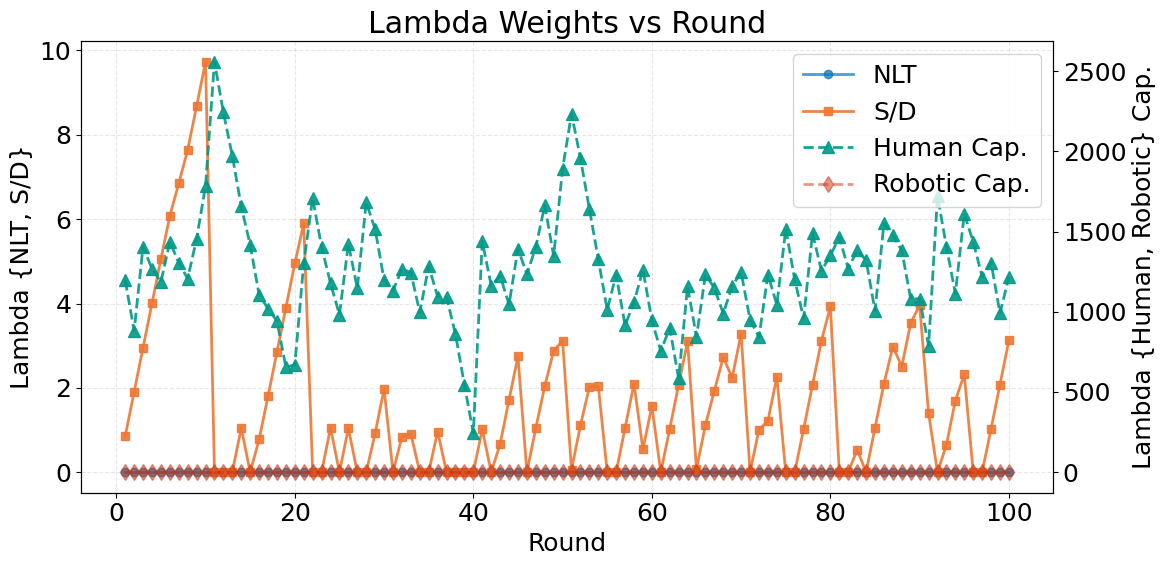}}
    \caption{Average Lagrange multipliers $\bar{\lambda}$ over training rounds. The $\nlt$ and robot capacity multipliers remain at zero, while the remaining constraints exhibit significant oscillations.}
    \label{fig:lambda_avg}
  \end{center}
\end{figure}

\subsection{Multi-Objective Optimization with MORL}

\begin{figure*}[t]
  \vskip 0.2in
  \centering
  \begin{subfigure}[t]{\linewidth}
    \centering
    \includegraphics[width=\linewidth]{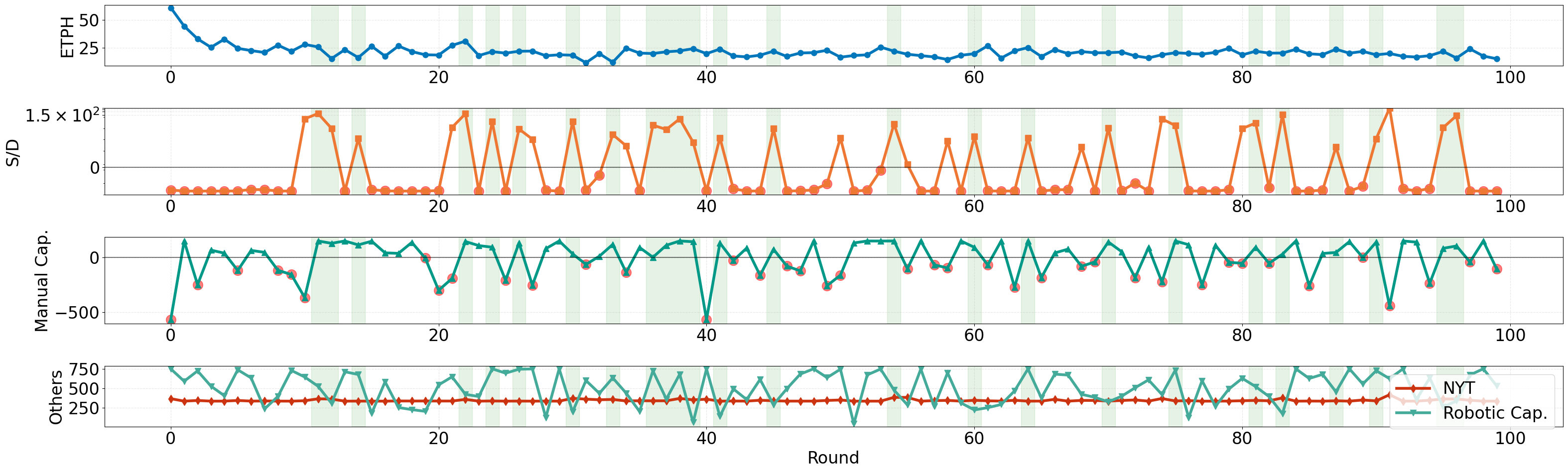}
    \caption{Single round policies}
    \label{fig:single_v_rounds_policy_perf}
  \end{subfigure}

  \vspace{0.15in}

  \begin{subfigure}[t]{\linewidth}
    \centering
    \includegraphics[width=\linewidth]{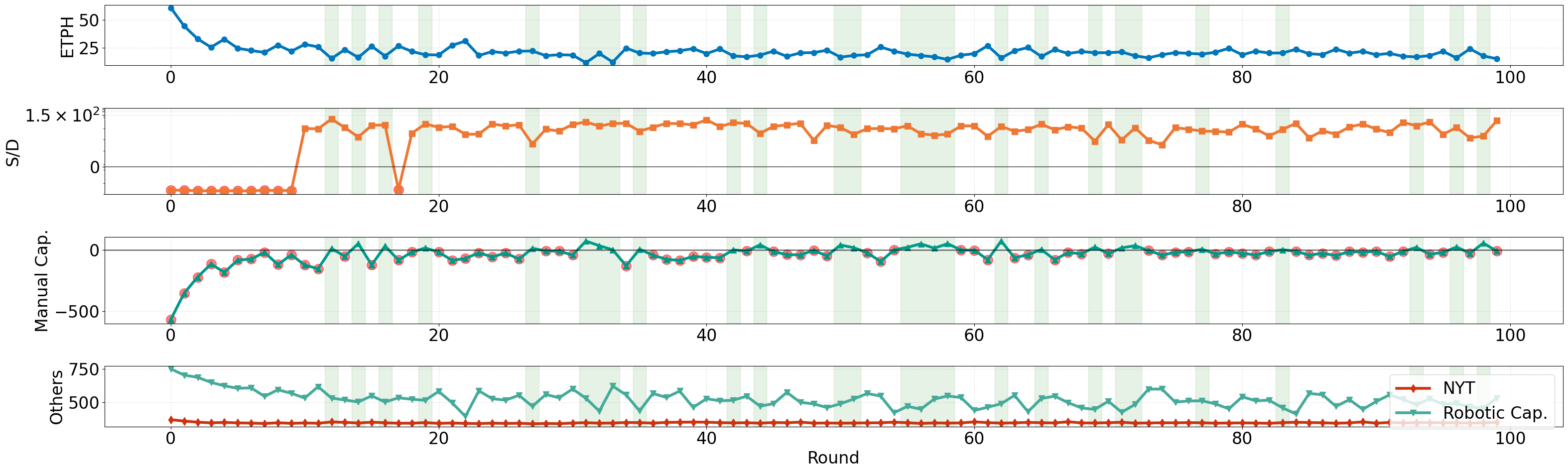}
    \caption{Time averaged policies}
    \label{fig:time_avg_v_rounds_policy_perf}
  \end{subfigure}

  \caption{
  Policy performance over repeated game rounds.
  (a) Performance of individual policies at each round.
  (b) Performance of time averaged policies.
  Each plot shows the global ETPH objective and four constraint values. Black horizontal lines denote constraint thresholds, red markers indicate violations, and green vertical regions indicate rounds where all constraints are simultaneously satisfied.
  }
  \label{fig:policy_perf_over_rounds}
\end{figure*}

As shown in \Cref{fig:lambda_avg,fig:single_v_rounds_policy_perf}, the regulator steers the learner through the Lagrange multipliers $\lambda_t$ in a repeated game between the learner and the regulator, with $C = 20{,}000$ in this experiment. Each round in \Cref{fig:single_v_rounds_policy_perf} corresponds to one interaction. In the objective and constraint plots, the horizontal line at $y = 0$ denotes zero constraint violation, with values $\geq 0$ indicating satisfaction and negative values indicating violation; for clarity, all quantities are shifted and plotted under the constrained formulation in \eqref{eq:contrained_rl}. In this experiment, the $\nlt$ and robot capacity constraints are largely satisfied across rounds, and the corresponding multipliers remain near zero, indicating that these constraints are typically inactive, whereas tightening the manual capacity and S/D constraints incurs a reduction in ETPH, reflecting the inherent trade off between throughput and constraint enforcement. The oscillations in $\lambda_t$ closely track transitions between constraint satisfaction and violation, highlighting the tight coupling between regulator updates and policy performance. Although feasibility is only guaranteed for the time averaged policy distribution, we empirically observe rounds in which the single policy learned at that round satisfies all constraints; we defer a detailed analysis of feasible stationary policies to \Cref{sec:feasible_single_policy}.

\begin{figure}[t]
  \vskip 0.2in
  \begin{center}
\centerline{\includegraphics[width=\linewidth]{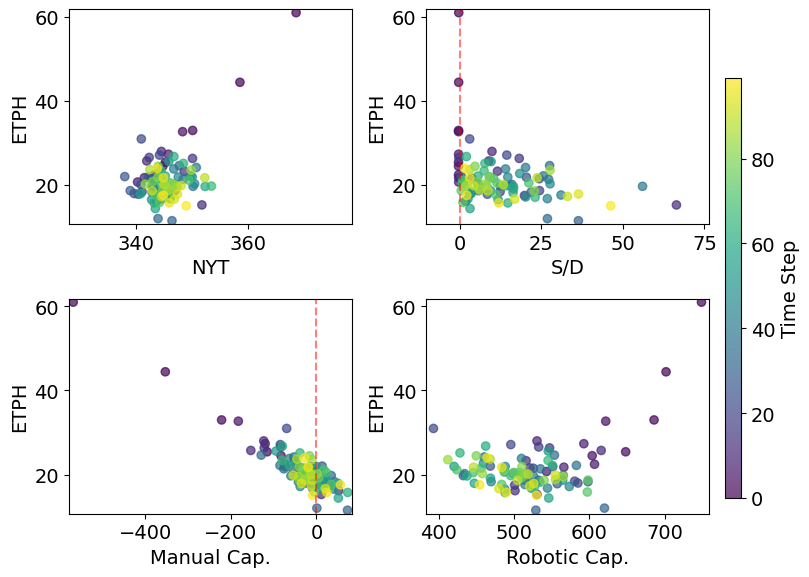}}
    \caption{Phase diagram of ETPH versus constraint satisfaction across rounds. Color indicates training rounds, red dashed lines denote constraint thresholds. As training progresses, time averaged policies trade ETPH for improved S/D and manual capacity satisfaction.}
    \label{fig:phase_diagram}
  \end{center}
\end{figure}

\begin{figure}[t]
  \vskip 0.2in
  \begin{center}
\centerline{\includegraphics[width=\columnwidth]{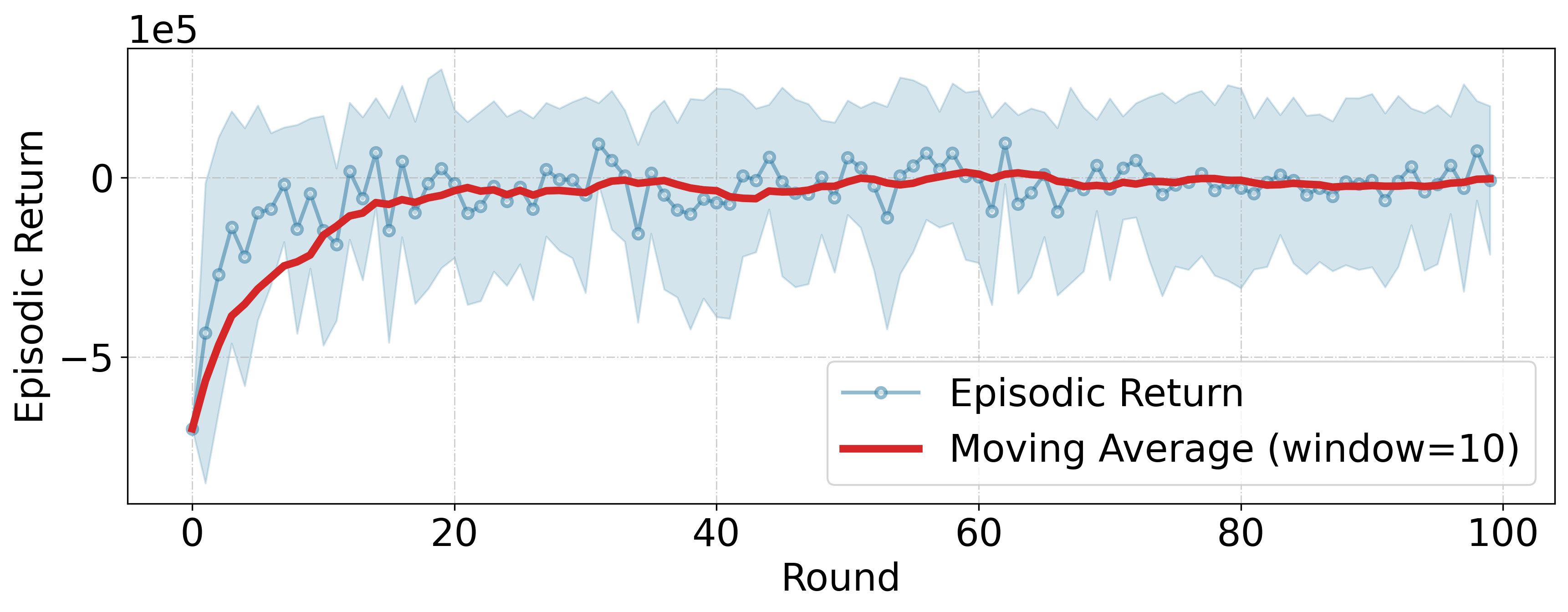}}
\caption{Average episodic return over rounds for time averaged policies parameterized by $\bar{\lambda}$, showing improved performance as policies are mixed across rounds.}
\label{fig:avg_episodic_return}
  \end{center}
\end{figure}

\cref{fig:time_avg_v_rounds_policy_perf} and \cref{fig:phase_diagram} illustrate the behavior of the time-averaged (mixture) policies produced by the repeated game dynamics. In line with the minimax and time-averaging guarantees established in \cref{sec:problem_setting_and_model}, the averaged policy distribution converges toward satisfying the S/D and manual capacity constraints, at the expense of some global ETPH. Consistent with this guarantee, \Cref{fig:avg_episodic_return} shows that, when evaluated with respect to the averaged multipliers $\bar{\lambda}$, the performance of the time-averaged policies improves as constraint satisfaction stabilizes. In practice, if exact feasibility is required, conservative slack can be introduced into the constraint thresholds. Given our theoretical guarantees on the near-minimax optimality of the time-averaged solution, such slack increases the probability that the resulting operating point satisfies the true desired constraints. 

\subsection{Feasible Stationary Policy via MORL}    \label{sec:feasible_single_policy}

\begin{figure}[t]
  \vskip 0.2in
  \begin{center}
    \centerline{\includegraphics[width=\columnwidth]{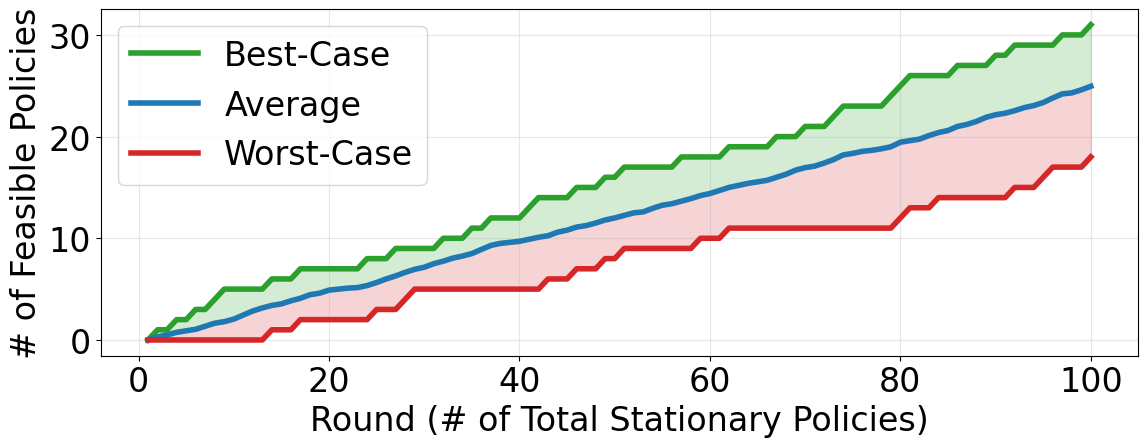}}
    \caption{Feasible stationary policies over training rounds. Although feasibility is guaranteed only in distribution, many individual policies are feasible in practice. Best, average, and worst over 20 seeds are shown.}
    \label{fig:feasible_stationary_policy_round}
  \end{center}
\end{figure}

\begin{table}[t]
  \caption{Comparison of ETPH and constraint satisfaction across MORL, random, and unconstrained policies. The unconstrained policy attains higher ETPH at the cost of severe capacity violations.}
  \label{tab:morl_v_baseline}
  \begin{center}
    \resizebox{\columnwidth}{!}{
      \begin{tabular}{lccc}
        \toprule
         & \makecell[c]{MORL \\ (single policy)} & \makecell[c]{Random \\ Actions} & Unconstrained \\
        \midrule
        \makecell[l]{Satisfies all \\ constraints?}
          & \textcolor{green}{$\surd$}
          & \textcolor{red}{$\times$}
          & \textcolor{red}{$\times$} \\
        \midrule
        ETPH
          & 20.52 & 9.19 & 61.81 \\

        $\nlt$ Slack
          & 362.62 & 359.15 & 370.62 \\

        \makecell[l]{S/D Slack}
          & 10.81 & 0.49 & \textcolor{red}{\bf{-0.37}} \\

        Manual Cap. Slack
          & 83.21
          & \textcolor{red}{\bf{-32.99}}
          & \textcolor{red}{\bf{-563.23}} \\

        Robot Cap. Slack
          & 258.23 & 571.01 & 743.68 \\
        \bottomrule
      \end{tabular}
    }
  \end{center}
  \vskip -0.1in
\end{table}

While our theoretical guarantees ensure feasibility only for the time-averaged policy distribution, we empirically observe that the repeated game dynamics often produce individual stationary policies that satisfy all constraints simultaneously. \Cref{fig:feasible_stationary_policy_round} summarizes this behavior by reporting the best, average, and worst cases over $20$ random seeds for the cumulative number of feasible single policies encountered across training rounds. Across runs, feasible stationary policies arise consistently, indicating that the learning dynamics frequently traverse regions of the policy space that satisfy all constraints, despite feasibility not being enforced at the level of individual iterates.

To assess the quality of such policies, \Cref{tab:morl_v_baseline} compares a representative feasible policy learned via MORL against random and unconstrained baselines. Although the unconstrained policy attains higher ETPH, it does so by severely violating capacity constraints, particularly manual capacity. In contrast, the MORL-derived policy achieves a balanced trade off, maintaining positive slack across all constraints while substantially outperforming the random baseline. Appendix \ref{app:err_canc} formalizes this empirical observation by introducing a modified theoretical framework that mitigates some of the effects of error cancellation and enables extraction of a single iterate from the sequence of learned iterates that achieves close to the value of the game.

\section{Conclusion}
We have explored a Multi-Objective RL framework for large-scale tote allocation and consolidation under realistic operational constraints. Our approach leverages theoretical results on best-response versus no-regret dynamics in zero-sum games, adapting them to the RL setting and introducing efficient implementation strategies to bridge the gap between theory and practice. Using a simulator that captures the core dynamics of fulfillment center operations, we demonstrate that our method is able to learn policies that effectively balance trade offs of KPIs.

This work opens several promising directions for future research. From a modeling perspective, exploring alternative MDP formulations and state/action space abstractions could yield further gains in scalability and realism. Algorithmically, initializing the learner with policies trained individually on each of the $m+1$ objectives, rather than retraining from scratch, may improve convergence speed and stability. Finally, incorporating strategic interactions and incentive structures between human and robotic agents presents an intriguing avenue for understanding and enhancing complementarity in hybrid fulfillment systems.


\clearpage

\bibliography{example_paper}
\bibliographystyle{icml2026}

\newpage
\appendix
\onecolumn

\section{Error Cancellations}
\label{app:err_canc}
Here, we will argue that if we have convergence of the randomized mixture of policies from the repeated game of no-regret vs. best-response dynamics from the learner and regulator, then there must be at least one policy in that set which we can select that approximately achieves the value of the game for this reformulated optimization problem. The broad intuition of this reformulation is that best-response vs. no-regret guarantees we satisfy the constraints on average for arbitrary convex constraints when feasible. The max of the constraint violations and $0$ is still convex if the constraints are individually convex. Now, the constraint value is never negative under this reformulation, so satisfying the constraint on average is only possible if we satisfy on most iterates and hence on some iterate.

\subsubsection{Reformulated Optimization Problem }
In the original formulation both negative and positive values of the constraint violation were included, which enabled a phenomenon known as error cancellation to occur where on average, the constraints are close to satisfied, but individual policies violate different constraints. Here, we reformulate the constraints using the positive part so that feasibility in expectation cannot be achieved via cancellations across different constraints, and violations of any single constraint are penalized uniformly.
  \begin{align*}
    \max \limits_{D \sim \Delta(\Pi)} &\; \E_{\pi \sim D} \Bigg[ \sum_{t=0}^{H-1} \etph_t(\pi) \Bigg] \\
    \text{s.t.} 
    &\quad  \max\Big\{\E_{\pi \sim D} \Bigg[ \sum_{t=0}^{H-1} N_{\texttt{large},t}(\pi) \Bigg]-\alpha_1,0\Big\} \leq 0, \\
    &\quad \max\Big\{\E_{\pi \sim D} \Bigg[ -\sum_{t=0}^{H-1} (S/D)_t(\pi) \Bigg]+\alpha_2,0\Big\} \leq 0, \\
    &\quad \max\Big\{\E_{\pi \sim D} \Bigg[ \sum_{t=0}^{H-1} \big(L^{H}_{S,t}(\pi)+L^{H}_{D,t}(\pi)\big) \Bigg]-\alpha_3,0\Big\}  \leq 0, \\
    &\quad \max\Big\{\E_{\pi \sim D} \Bigg[ \sum_{t=0}^{H-1} \big(L^{R}_{S,t}(\pi)+L^{R}_{D,t}(\pi)\big) \Bigg]-\alpha_4,0\Big\} \leq 0,
\end{align*}
This can be even further reduced to 
  \begin{align*}
    \max \limits_{D \sim \Delta(\Pi)} &\; \E_{\pi \sim D} \Bigg[ \sum_{t=0}^{H-1} \etph_t(\pi) \Bigg] \\
    \text{s.t.} 
    &\quad  \max\Big\{\E_{\pi \sim D} \Bigg[ \sum_{t=0}^{H-1} N_{\texttt{large},t}(\pi) \Bigg]-\alpha_1,\E_{\pi \sim D} \Bigg[ -\sum_{t=0}^{H-1} (S/D)_t(\pi) \Bigg]+\alpha_2,\\&\E_{\pi \sim D} \Bigg[ \sum_{t=0}^{H-1} \big(L^{H}_{S,t}(\pi)+L^{H}_{D,t}(\pi)\big) \Bigg]-\alpha_3,\E_{\pi \sim D} \Bigg[ \sum_{t=0}^{H-1} \big(L^{R}_{S,t}(\pi)+L^{R}_{D,t}(\pi)\big) \Bigg]-\alpha_4,0\Big\} \leq 0,
\end{align*}

\subsubsection{Algorithm}

\begin{algorithm}[H]
\caption{Repeated Game: Learner vs Regulator}
\begin{algorithmic}[1]
\STATE \textbf{Initialize:} $\lambda_0 \in \Lambda$, total rounds $T$, step size $\eta$
\STATE $\bar{D} \leftarrow 0$, $\bar{\lambda} \leftarrow 0$ \COMMENT{Time averaged strategies}
\FOR{$t = 1$ to $T$}
    \STATE \textbf{Learner:} $D_t \leftarrow \mathrm{DQNBestResponse}(\lambda_{t-1})$
    \STATE $g_t \leftarrow \mathrm{EvaluatePolicy}(D_t)$ \COMMENT{Constraint slacks}
    \STATE \textbf{Regulator:} $\lambda_t \leftarrow \mathrm{Proj}_{\Lambda}(\lambda_{t-1} - \eta g_t)$ \COMMENT{Online gradient descent}
    \STATE $\bar{\lambda} \leftarrow \frac{1}{t}\bigl((t-1)\bar{\lambda} + \lambda_t\bigr)$
\ENDFOR
\STATE \textbf{Return:} $\mathrm{Best}\bigl(\{D_t\}_{t \in [T]}, \bar{\lambda}\bigr)$ \COMMENT{Best single strategy under $\bar{\lambda}$}
\end{algorithmic}
\end{algorithm}


\subsubsection{Finding the Estimated Best Iterate}
We select the best $D_t$ from the algorithm based upon an estimated Lagrangian with respect to $\bar{\lambda}$.  Let $L(D,\lambda)=\E_{\pi\sim D}\big[V_0^\pi(\mu)\big]\;-\;\lambda\,[\,g(D)\,]_+$,
    where $g(D):=\max_{i\in[m]}(\E_{\pi \sim D}[V_i^\pi(\mu)]-\alpha_i)$ and $\lambda\ge 0$. In this presentation, it is worth noting that we are assuming that without loss of generality (WLOG) we can formulate our constrained optimization problem to be of this form and that all signs are taken into account in the assignment of $\alpha_i$. Note that we could modify our reward/value function bounds to take into account possible negative values, but this would not change the analysis our bounds much so we use the following WLOG formulation for clarity. We are interested in finding $\arg \max_{t \in [T]} L(D_t,\bar{\lambda})$. However, we do not have an exact estimate of the Lagrangian. Instead, we compute $\hat{L}$ as follows: suppose we form Monte Carlo estimates from $n$ i.i.d. episodes per iterate: $\widehat V_0(D_t)=\tfrac1n\sum_{j=1}^n \hat V_0^j(\tau_{t,j})$, $\widehat c_i(D_t)=\tfrac1n\sum_{j=1}^n \hat V_i^j(\tau_{t,j})$, $\widehat g(D_t):=\max_i(\widehat c_i(D_t)-\alpha_i)$, and $\widehat L(D_t,\bar\lambda):=\widehat V_0(D_t)-\bar\lambda\,[\widehat g(D_t)]_+$. 

\subsubsection{Learner's Best-Response Convex RL}
Earlier we discussed approximate best-response for the learner in terms of the ability to approximately solve single-objective RL problems. In this section, we discuss this notion when the objective functions are concave in the variable $D$ corresponding to the distribution over policies. 


\begin{definition}[Episodic Occupancy Measure]
    \begin{align*}
        q_h^\pi(s,a) &= \mathrm{Pr}^{\pi,\mu}(s_h=s,a_h=a)
    \end{align*}

\end{definition}
Define $d= \sum_{h=0}^{H-1} q_h^{\pi}$. In practice, we may have a policy $\pi$ specified by a $D \in \Delta(\Pi)$, so we may use the shorthand notation $\mu_D$ to denote this summed over time occupancy measure over the given policy. In an episodic RL setting especially with fixed rewards over time, we can denote $\E_{\pi \sim D}[V^{\pi}(\mu)]= \langle r,\mu_D\rangle$ (or $\langle r,d\rangle$ using our notation above) as will be shown below.

\begin{align*}
    V_i^\pi(\mu) &= \E_{\pi,\mu}\Big[\sum_{h=0}^{H-1} r(s_h,a_h)\Big] \\
    &=  \E_{\pi,\mu}\Big[\sum_{h=0}^{H-1} \sum_{s,a} r(s,a)I(s_h=s,a_h=a)\Big] \\
    &= \sum_{h=0}^{H-1} \sum_{s,a} r(s,a) \E_{\pi}[I(s_h=s,a_h=a)] \\
    &= \sum_{h=0}^{H-1} \sum_{s,a} r(s,a) \mathrm{Pr}^{\pi,\mu}(s_h=s,a_h=a) \\
    &= \sum_{h=0}^{H-1} \sum_{s,a} r(s,a) q_h^{\pi}(s,a) \\
    &= \sum_{s,a} r(s,a) \sum_{h=0}^{H-1} q_h^{\pi}(s,a) \\
    &= \langle r, d \rangle
\end{align*}

Let \(r(s,a)\) denote the per-step \emph{reward}. For each constraint \(i=1,\dots,m\), let \(r_i(s,a)\) denote the per-step \emph{constraint signal}. For a distribution over policies \(D\) with occupancy measure \(\mu_D\), define 
\[
r_i(D):=\langle \mu_D, r_i\rangle=\E_{\pi\sim D}\Big[\sum_{h=0}^{H-1} r_i(s_h,a_h)\Big].
\]
Given thresholds \(\alpha_i\), the \emph{constraint violation} for constraint \(i\) is \(g_i(D):=r_i(D)-\alpha_i\), and the aggregate violation is
\[g(D):=\max_{i\in[m]} g_i(D),\qquad [x]_+:=\max\{0,x\}.
\]
The Lagrangian at multiplier \(\lambda\ge 0\) is thus
\[L(D,\lambda)=r(D)-\lambda\,[\,g(D)\,]_+ .\]
For fixed $\lambda_t\ge 0$ and current iterate $D_t$, write
\[L(D,\lambda_t)=\langle\mu_D,r\rangle-\lambda_t\,[\max_i\{\langle\mu_D,r_i\rangle-\alpha_i\}]_+ .\]
Because $D\mapsto L(D,\lambda_t)$ is concave, any supergradient at $D_t$ yields an \emph{affine upper bound}. 

More directly, for a given $\lambda_t \in \Lambda$, we can write 
\[f(d) = L(d,\lambda_t)=\langle d,r\rangle-\lambda_t\,[\max_i\{\langle d,r_i\rangle-\alpha_i\}]_+ .\]

Let
\begin{align*}
    \partial_d \mathcal{L}(d,\lambda) = \begin{cases}
        r_0 \quad \text{if } [g(d)]_+ \leq 0\\
        r_0 - \lambda_t \sum_j p_j r_j \text{if } [g(d)]_+ \geq 0 \text{ where } j \in \arg\max g_i(d) \text{ and } p_j \geq 0 \text{ } \forall j, \sum_j p_j = 1
    \end{cases}
\end{align*}
In practice, to evaluate this we would take the given policy specified by $D$, evaluate the performance of the policy over all the constraints, then compute the gradient based upon the above.


Then with $s_t:= r-\lambda_t\sum_i p_{t,i}\,r_i$ where $p_t$ is set according to the supergradient, we have, for all $D$,
\[L(D,\lambda_t)\ \le\ L(D_t,\lambda_t)+\langle\mu_D-\mu_{D_t},\ s_t\rangle.\]

(Maximizing the linear form $\langle\mu_D,s_t\rangle$ is the standard \emph{linear minimization oracle} / RL call used in Frank--Wolfe and convex RL \cite{jaggi2013fw, zahavy2021reward, geist2021concave}. )
Thus the learner’s \emph{$\varepsilon$-best-response} can be implemented by solving the single linear problem
\[\max_{D}\ \langle\mu_D,\ s_t\rangle\quad\Longleftrightarrow\quad \max_{D}\ \E_{\pi \sim D}\Big[\sum_h\,\underbrace{\big(r(s_h,a_h)-\lambda_t\sum_i p_{t,i}r_i(s_h,a_h)\big)}_{\displaystyle r_{\lambda_t}(s_h,a_h)}\Big]\]
up to additive error $\varepsilon$. Therefore, we have that if $\langle\mu_{D^*},s_t\rangle-\langle\mu_{D_{t+1}},s_t\rangle\le\varepsilon$, then 
\[\max_D L(D,\lambda_t)-L(D_{t+1},\lambda_t)\ \le\ \varepsilon.\]

\subsubsection{Frank Wolfe Procedure}
For every $\lambda$ selected by the regulator, the learner will have to solve a sequence of $RL$ problems over linearized reward functions in order to approximately best respond to the regulator's strategy. This will be done by computing super-gradients of the Lagrangian. In this setting, though the objective is concave in $D$ and no longer linear, each $D_t$ corresponds to a representative policy based upon the single policies returned from the FW linear oracles.

\begin{algorithm}[H]
\caption{Learner Concave Best-Response}
\begin{algorithmic}[1]
\STATE \textbf{Receive:} $\lambda \in \Lambda$, step size $\alpha$, tolerance $\epsilon$
\STATE \textbf{Initialize occupancy measure:} $x^{(1)} \in \mathcal{O}$ \COMMENT{Current iterate (occupancy measure)}
\FOR{$w = 1$ to $W$}
    \STATE $\alpha_w \leftarrow \frac{2}{1+w}$
    \STATE $d^{(w)} \leftarrow \arg\max_{d \in \mathcal{O}} \left\langle \partial \mathcal{L}\!\left(x^{(w)}\right),\, d \right\rangle$
    \COMMENT{DQN with scalarized reward induced by $\lambda$}
    \STATE $g^{(w)} \leftarrow \left\langle \partial \mathcal{L}\!\left(x^{(w)}\right),\, d^{(w)} - x^{(w)} \right\rangle$
    \IF{$g^{(w)} \le \epsilon$}
        \STATE \textbf{break}
    \ENDIF
    \STATE $x^{(w+1)} \leftarrow (1-\alpha_w)\,x^{(w)} + \alpha_w\, d^{(w)}$ \COMMENT{Convex update}
\ENDFOR
\STATE \textbf{Compute policy:} $\pi(s,a) \leftarrow \dfrac{x^{(w_\star)}(s,a)}{\sum_{a' \in \mathcal{A}} x^{(w_\star)}(s,a')}$ \COMMENT{$w_\star$ is final iterate}
\STATE \textbf{Assign} $D \in \Delta(\Pi)$ accordingly
\STATE \textbf{Return:} $D$ \COMMENT{Learner best-response}
\end{algorithmic}
\end{algorithm}

\subsubsection{Theoretical Results}

We now present our theoretical results on the single round's $D_t$ selected out of this procedure.

Let $V_i^\pi(\mu)$ denote the (episodic) return for objective/constraint $i$ under policy $\pi$ from initial distribution $\mu$.
Define the per-policy worst-constraint violation
\[
g(D)\ :=\ \max_{i\in[m]}\Big(\E_{\pi \sim D}[V_i^\pi(\mu)]-\alpha_i\Big),\qquad [x]_+:=\max\{0,x\}.
\]
For a distribution $D\in\Delta(\Pi)$ over policies and scalar multiplier $\lambda\in[0,C]$, define the \emph{reformulated} Lagrangian
\begin{equation}\label{eq:reform-lagrangian}
L(D,\lambda)\ :=\ \E_{\pi \sim D}[V_0^\pi(\mu)]\ -\ \lambda\,[g(D)]_+.
\end{equation}

Let $L^\star:=\min_{\lambda\in[0,C]}\max_{D\in\Delta(\Pi)} L(D,\lambda)$ be the minimax value of the reformulated game.
Assume the repeated game dynamics produce a time-average pair $(\bar D,\bar\lambda)$ that is a $\nu$-approximate minimax equilibrium in the sense of \cref{def:minimax_equilibrium} (applied to~\eqref{eq:reform-lagrangian}). Papers like \cite{eaton2025intersectionalfairnessreinforcementlearning} and \cite{pmlr-v80-kearns18a} have already proven that the time-averaged solution satisfies this property, so we are building on the basis of those results and this is not truly a separate assumption. We simply abstract away that analysis since it has been done elsewhere. In particular, we will use the learner-side condition
\begin{equation}\label{eq:nu-minimax-used}
L(\bar D,\bar\lambda)\ \ge\ \max_{D\in\Delta(\Pi)} L(D,\bar\lambda)\ -\ \nu\ \ge\ L^\star-\nu.
\end{equation}

We assume constraint thresholds satisfy $\alpha_i \in [-H,H]$.
Since per-step rewards lie in $[0,1]$ and the horizon is $H$,
this implies $V_i^\pi(\mu)\in[0,H]$ and hence
\[
g_i(D) = \mathbb{E}_{\pi \sim D}[V_i^\pi(\mu)] - \alpha_i \in [-H,\,2H],
\qquad
[g_i(D)]_+ \in [0,\,2H].
\]
\begin{lemma}[Uniform Concentration for $\widehat L$]\label{lem:uniform_conc}
Assume per-step rewards and per-step constraint signals are bounded in $[0,1]$, so that for every policy $\pi$,
\[
0\le V_0^\pi(\mu)\le H,\qquad 0\le V_i^\pi(\mu)\le H\ \ (i\in[m]),
\]
and assume $\alpha_i\in[-H,H]$ for all $i$, so that
\[
0\le [g(D)]_+\le 2H.
\]
Fix $\bar\lambda\in[0,C]$ and iterates $\{D_t\}_{t=1}^T$.
For each $t$, form $\widehat L(D_t,\bar\lambda)$ from $n$ i.i.d. rollouts of a policy $\pi\sim D_t$
(sampling $\pi$ once per rollout, then generating one episode),
using the unbiased Monte Carlo estimates of $V_0^\pi(\mu)$ and $V_i^\pi(\mu)$ to compute estimates of $[g(D)]_+$ and $V_0^\pi(\mu)$.
Then for any $\varepsilon,\delta\in(0,1)$,
\[
\Pr\!\left(\max_{t\in[T]}\big|\widehat L(D_t,\bar\lambda)-L(D_t,\bar\lambda)\big|\ \le\ \varepsilon\right)\ \ge\ 1-\delta
\]
whenever
\begin{equation}\label{eq:n_needed_clean}
n\ \ge\ \frac{(1+2\bar\lambda)^2H^2}{2\varepsilon^2}\,\ln\!\frac{2T}{\delta}.
\end{equation}
\end{lemma}

\begin{proof}
For a single rollout of $(\pi,\tau)$ with $\pi\sim D_t$, define the bounded random variable
\[
Z:=V_0^\pi(\mu)-\bar\lambda\,[g(\pi)]_+\in[-2\bar\lambda H,\,H],
\]
so the range is at most $(1+2\bar\lambda)H$.
The estimator $\widehat L(D_t,\bar\lambda)$ is the empirical mean of $n$ i.i.d. copies of $Z$, hence Hoeffding's inequality gives
\[
\Pr\big(|\widehat L(D_t,\bar\lambda)-L(D_t,\bar\lambda)|>\varepsilon\big)
\ \le\ 2\exp\!\left(-\frac{2n\varepsilon^2}{(1+2\bar\lambda)^2H^2}\right).
\]
A union bound over $t\in[T]$ yields the claim under~\eqref{eq:n_needed_clean}.
\end{proof}
Define the iterate–average Jensen gap at $\bar\lambda$:
\begin{equation}\label{eq:Javg-def}
J_{\texttt{avg}}(\bar\lambda)\;:=\;
L(\bar D,\bar\lambda)\;-\;\frac1T\sum_{t=1}^T L(D_t,\bar\lambda)\ \ \ge 0,
\end{equation}
which is nonnegative because $D\mapsto L(D,\bar\lambda)$ is concave. 
\begin{theorem}[Single Iterate Extraction]\label{thm:single_policy_extraction}
Assume $(\bar D,\bar\lambda)$ satisfies~\eqref{eq:nu-minimax-used} for the reformulated Lagrangian~\eqref{eq:reform-lagrangian}.
Let $t^\star\in\arg\max_{t\in[T]}\widehat L(D_t,\bar\lambda)$, using $\widehat L$ as in \cref{lem:uniform_conc}.
Then with probability at least $1-\delta$,
\begin{equation}\label{eq:iterate_value_bound}
L(D_{t^\star},\bar\lambda)\ \ge\ L^\star\ -\ (\nu+2\varepsilon+J_{\texttt{avg}}(\bar{\lambda})).
\end{equation}
Moreover, there exists a policy $\pi^\star$ in the support of $D_{t^\star}$ such that
\begin{equation}\label{eq:policy_value_bound}
L(D_{\pi^\star},\bar\lambda)\ \ge\ L^\star\ -\ (\nu+2\varepsilon+J_{\texttt{avg}}(\bar{\lambda})).
\end{equation}
\end{theorem}

\begin{proof}
On the uniform concentration event of \cref{lem:uniform_conc}, for all $t$ we have
\[
L(D_t,\bar\lambda)\ \le\ \widehat L(D_t,\bar\lambda)+\varepsilon
\ \le\ \widehat L(D_{t^\star},\bar\lambda)+\varepsilon
\ \le\ L(D_{t^\star},\bar\lambda)+2\varepsilon,
\]
so $\max_t L(D_t,\bar\lambda)\le L(D_{t^\star},\bar\lambda)+2\varepsilon$ and hence
\begin{equation}\label{eq:best_iter_close_to_max}
L(D_{t^\star},\bar \lambda)\ \geq\ \max_{t\in[T]} L(D_t,\bar\lambda)\ -\ 2\varepsilon.
\end{equation}
Because $L(\cdot,\bar\lambda)$ is concave in $D$ and by definition of the Jensen gap,
\[
L(\bar D,\bar\lambda) \le\ \frac1T\sum_{t=1}^T D_t +J_{\texttt{avg}}(\bar{\lambda})\le \max_{t\in[T]} L(D_t,\bar\lambda)+J_{\texttt{avg}}(\bar{\lambda}).
\]
Combining with~\eqref{eq:best_iter_close_to_max} and~\eqref{eq:nu-minimax-used} yields
\[
L(D_{t^\star},\bar\lambda)\ \ge\ L(\bar D,\bar\lambda)-2\varepsilon-J_{\texttt{avg}}(\bar{\lambda})\ \ge\ L^\star-(\nu+2\varepsilon+J_{\texttt{avg}}(\bar{\lambda})),
\]
which is~\eqref{eq:iterate_value_bound}.

Note that the FW oracle is already returning a representative single policy, but another way to see the single policy statement is as follows. For the single-policy statement, expand~\eqref{eq:reform-lagrangian} at $(D_{t^\star},\bar\lambda)$:
\[
L(D_{t^\star},\bar\lambda)\ =\ \E_{\pi\sim D_{t^\star}} L(\pi,\bar\lambda).
\]
Therefore there exists $\pi^\star$ in the support of $D_{t^\star}$ with $L(D_{\pi^\star},\bar\lambda)\ge L(D_{t^\star},\bar\lambda)$, giving~\eqref{eq:policy_value_bound}.
\end{proof}
While the Jensen gap can be nonzero due to concavity, in our setting the Lagrangian is piecewise linear in D, making the gap empirically smaller than the worst-case. To analyze constraint violations, we can use a bound on the average constraint violation level coupled with non-negativity to reason about the constraint violation level of a single iterate.

\end{document}